\documentclass[lettersize,journal]{IEEEtran}
\usepackage{amsmath,amsfonts}
\usepackage{amsthm}
\usepackage{bm}
\usepackage{algorithm}
\usepackage[noend]{algpseudocode}
\usepackage{graphicx}
\usepackage{array}
\usepackage[caption=false,font=normalsize,labelfont=sf,textfont=sf]{subfig}
\usepackage{textcomp}
\usepackage{stfloats}
\usepackage{url}
\usepackage{verbatim}
\usepackage{graphicx}
\usepackage{subfig} 
\usepackage{cite}
\usepackage{booktabs}
\usepackage{siunitx}
\usepackage{balance}
\usepackage{color,xcolor}
\definecolor{maincolor}{HTML}{00274C}
\usepackage{tikz}
\usepackage[utf8]{inputenc} 
\usetikzlibrary{decorations.pathreplacing,arrows,shapes,positioning,shadows,calc}
\usetikzlibrary{decorations, decorations.text,backgrounds}
\tikzset{every picture/.style={font issue=\footnotesize},
    font issue/.style={execute at begin picture={#1\selectfont}}
}

\usepackage{latexsym,bm}

\usepackage{tabu}
\usepackage{multirow}
\usepackage{lipsum}
\usepackage{multicol}
\usepackage{hyperref}
\hypersetup{
    colorlinks=true,
    linkcolor=blue,
    citecolor=magenta,      
    urlcolor=cyan,
    pdfpagemode=FullScreen,
    }
\usepackage{tikz}
\usepackage{graphicx}
\usepackage{fontawesome}
\usepackage{multirow}
\usepackage{pgfplots}
\usepgfplotslibrary{fillbetween}
\pgfplotsset{compat=1.7}
\usetikzlibrary {arrows.meta}
\usetikzlibrary {patterns,patterns.meta}

\usetikzlibrary{decorations.text, decorations.pathmorphing, shapes}
\usetikzlibrary{arrows}
\usetikzlibrary{arrows.meta, decorations.pathmorphing}
\usetikzlibrary{arrows,shapes,decorations.text}
\usepackage{marvosym}

\newtheoremstyle{ieeetheorem}
{}                
{}                
{\upshape}        
{}                
{\bfseries}       
{.}               
{ }               
{}                
\theoremstyle{ieeetheorem}
\newtheorem{theorem}{Theorem}
\theoremstyle{definition}

\pgfmathdeclarefunction{laplace}{2}{%
  \pgfmathparse{-1/(2*#2)*exp(-abs(x-#1)/#2)}%
}
\definecolor{maincolor}{HTML}{032F99}
\definecolor{blue}{RGB}{31,64,122}
\definecolor{red}{HTML}{e05a87} 

\makeatletter
\newcommand{\oset}[3][0ex]{%
  \mathrel{\mathop{#3}\limits^{
    \vbox to#1{\kern-2\ex@
    \hbox{$\scriptstyle#2$}\vss}}}}
\makeatother

\newcommand{\mb}[1]{\mathbf{#1}}
\newcommand{\mbg}[1]{\boldsymbol{#1}}
\begin{document}

\title{Constrained Diffusion Models for Synthesizing Representative Power Flow Datasets}

\author{Milad Hoseinpour,~\IEEEmembership{Student Member,~IEEE}, \ Vladimir Dvorkin,~\IEEEmembership{Member,~IEEE}
\thanks{
The authors are with the Department of Electrical Engineering and
Computer Science, University of Michigan, Ann Arbor, MI
48109, USA. E-mail: \{miladh,~dvorkin\}@umich.edu.
}
}



\maketitle

\begin{abstract}
High-quality power flow datasets are essential for training machine learning models in power systems. However, security and privacy concerns restrict access to real-world data, making statistically accurate and physically consistent synthetic datasets a viable alternative. We develop a diffusion model for generating synthetic power flow datasets from real-world power grids that both replicate the statistical properties of the real-world data and ensure AC power flow feasibility. To enforce the constraints, we incorporate gradient guidance based on the power flow constraints to steer diffusion sampling toward feasible samples. For computational efficiency, we further leverage insights from the fast decoupled power flow method and propose a variable decoupling strategy for the training and sampling of the diffusion model. These solutions lead to a physics-informed diffusion model, generating power flow datasets that outperform those from the standard diffusion in terms of feasibility and statistical similarity, as shown in experiments across IEEE benchmark systems. 
\end{abstract}


\section{Introduction}
\IEEEPARstart{P}{ower} flow datasets \cite{klamkin2025pglearn,gillioz2025large,VENZKE2021106614} are essential for training and benchmarking machine learning (ML) models for optimal power flow (OPF) \cite{van2021machine} and state estimation \cite{pagnier2021embedding}. However, the real-world power flow datasets are rarely available due to privacy, security, and legal barriers \cite{joswig2022opf,bugaje2023generating,hoseinpour2023privacy, dvorkin2023differentially,wu2025synthesizing}. Recent advances in generative AI, capable of producing synthetic data with distributions similar to the original data \cite{
zhang2024generating,wang2022generating,pan2019data,Wang_2022,larsen2016autoencoding,zhang2018generative,gu2019gan,fekri2019generating,el2020data,chen2018model
}, have partially lifted these barriers, yet statistical consistency alone cannot guarantee adherence to physical grid constraints \cite{
hoseinpour2025domain}. Consequently, ML models trained on constraint-agnostic synthetic datasets are likely to perform substantially worse than those trained on original data. This paper introduces a data generation framework to synthesize statistically consistent and physically meaningful power flow datasets. To achieve this, we develop a constrained diffusion model to learn the underlying distribution of power flow data and generate synthetic samples that are both statistically representative and feasible with respect to the AC power flow constraints. This constrained diffusion model can be trained internally by system operators to publicly release high-quality synthetic power flow data to support a wide range of downstream ML applications.

\subsection{Related Work}
The literature on generating synthetic datasets for power systems broadly falls into two categories: generic random sampling and historical data-driven approaches. 

The former focuses on power flow data generation through iterative uniform sampling of loads followed by solving the OPF problem \cite{singh2021learning,fioretto2020predicting}. 
In \cite{zamzam2020learning}, authors use a truncated Gaussian distribution as another variation of sampling, which also accounts for correlations between power injections at different locations. However, the datasets based on generic sampling only represent a small portion of the feasibility region. To solve this, \cite{joswig2022opf} uniformly samples loads from a convex set, containing the feasible region, and iteratively refines this set using infeasibility certificates. In \cite{nadal2023scalable}, a bilevel optimization is proposed to sample operating conditions close to the boundaries of the feasible region, which is more informative that a random sampling. A basic requirement for ML-based OPF solvers is robustness to grid topology variations, e.g., network topology switching \cite{popli2024robustness}. To meet this requirement, authors in \cite{lovett2024opfdata} incorporate topological perturbations in addition to load perturbations in their synthetic data generation framework. 

Although straightforward, random sampling comes with certain limitations. The resulting datasets do not represent the true underlying distribution of real-world operating conditions. That is, the synthetic data points may fail to capture correlations, patterns, or variability present in historical data. ML-based OPF solvers trained on such data may generalize poorly, leading to inaccurate predictions and erroneous uncertainty quantification \cite{kiyani2025decision,shafer2008tutorial}. Moreover, the required number of random samples to cover the whole feasible region grows exponentially in the size of the grid \cite{bengio2013representation,stiasny2022closing}.

The historical data-driven approaches, instead, learn the underlying data distribution from real operational records. This approach has been enabled by advances in generative models, such as variational autoencoders (VAEs), generative adversarial networks (GANs), and diffusion models. For instance, the VAE from \cite{pan2019data} generates synthetic electric vehicles load profiles, and conditional VAE from \cite{wang2022generating} does the same for snapshots of multi-area electricity demand. Reference \cite{Wang_2022} presents conditional VAE for synthesizing load profiles of industrial and commercial customers, conditioned on time and typical power exchange with the grid. However, VAEs may struggle with complex and high-dimensional datasets and result in low quality samples \cite{larsen2016autoencoding}. Moreover, there is no principled approach for VAEs to control the generated outputs, making it difficult to enforce domain-specific constraints. GANs have also been used to synthesize load patterns in power systems \cite{zhang2018generative,gu2019gan,fekri2019generating,el2020data,chen2018model}. For instance, \cite{el2020data}  proposes a GAN model to generate synthetic appliance-level load patterns and usage habits. In \cite{chen2018model}, authors propose a GAN-based framework for renewable energy scenario generation that effectively captures both temporal and spatial patterns across a large number of correlated resources. Nonetheless, GANs also suffer from issues such as training instability, mode collapse, and the lack of principled means for controllability \cite{jabbar2021survey}. 

Addressing the limitations of GANs and VAEs, diffusion models have emerged as the leading choice for generative models \cite{dhariwal2021diffusion}. A physics-informed diffusion model is proposed in \cite{zhang2024generating} for generating synthetic net load data, where the solar PV system performance model is embedded into the diffusion model. In \cite{dong2023short}, authors propose a conditional latent diffusion model for short-term wind power scenario generation, which uses weather conditions as inputs. Authors in \cite{li2024diffcharge} developed a framework based on diffusion models to generate electric vehicle charging demand time-series data, which is also capable of capturing temporal correlation between charging stations.

While this line of work advocates for diffusion models to generate power systems data, it primarily focuses on statistical consistency. To the best of our knowledge, no prior work has explored the integration of domain constraints such as the AC power flow constraints directly into the diffusion process.  

\subsection{Summary of Contributions}
The main contribution of this paper is a generative AI framework that leverages power systems operational data to synthesize credible power flow datasets for ML applications. Specific contributions are summarized as follows:
\begin{enumerate}
    \item We develop a diffusion model capable of generating high-quality power flow datasets that inherit the statistical properties of the actual power flow records. Unlike random sampling in \cite{singh2021learning,fioretto2020predicting,zamzam2020learning,joswig2022opf,nadal2023scalable,popli2024robustness,lovett2024opfdata}, the model does not require any distributional assumptions; rather, as a generative model, it learns the distribution of power flow data directly from historical records. Yet, unlike the existing generative models in \cite{zhang2024generating,pan2019data,wang2022generating,Wang_2022,larsen2016autoencoding,zhang2018generative,gu2019gan,fekri2019generating,el2020data,chen2018model}, it controls the output to ensure compliance of synthetic samples with the grid physics. 
    Our model focuses on synthesizing power injections and voltage variables---the data to support OPF, state estimation, and other applications of ML to power systems optimization and control.  
    \item We introduce a guidance term within the sampling phase to ensure that the synthesized data points are feasible with respect to the AC power flow constraints. The guidance term corresponds to a single iteration of Riemannian gradient descent on the clean data manifold---the space of all physically valid power flow states. We formally prove that this approach improves physical consistency without pushing samples off the learned distribution. As a result, the generated power flow data points are both feasible and statistically representative.
    
    \item We leverage power systems domain knowledge for implementing the diffusion model. Inspired by the classical fast decoupled power flow method, we decouple the full variable vector and use two smaller denoiser neural networks to improve scalability. We then propose a custom normalization of the AC power flow equations for stabilizing the sampling of power flow variables. 
\end{enumerate}

The remainder is organized as follows. The problem statement is presented in Sec. \ref{problem_setup}, followed by Sec. \ref{Preliminaries} with preliminaries on diffusion models and power flows. Section \ref{gradient_guidance} introduces the proposed manifold-constrained guidance for enforcing the power flow constraints in diffusion sampling. Then, Sec. \ref{implementation_strategies} provides implementation insights tailored to power systems: variable decoupling and normalization for scale-consistent gradient guidance. Section \ref{results} provides numerical results on the standard IEEE test cases. Section \ref{conclusion} concludes. 

\section{Problem Statement}\label{problem_setup}

Consider a power grid characterized by vectors of active power injections $\mb{p}$, reactive power injections $\mb{q}$, voltage magnitudes $\mb{v}$, and phase angles $\mbg{\theta}$. Given a historical dataset $\mathcal{D}=\{(\mb{p}_{i},\mb{q}_{i},\mb{v}_{i},\mbg{\theta}_{i})\}_{i=1}^{N}$ with $N$ power flow records, our goal is to generate a synthetic dataset $\widetilde{\mathcal{D}}=\{(\tilde{\mb{p}}_{i},\tilde{\mb{q}}_{i},\tilde{\mb{v}}_{i},\tilde{\mbg{\theta}}_{i})\}_{i=1}^{M}$ with $M$ records, which are statistically representative of the given dataset $\mathcal{D}$ and are feasible with respect to the power flow constraints, i.e., data should satisfy the following conditions:
\begin{subequations}
\begin{align}
    &\min\textbf{dist}\left(p_{\text{syn}}~||~p_{\text{real}}\right),
    &&\label{statistical_distance}\\
    &\mathcal{G}(\tilde{\mb{p}}_{i},\tilde{\mb{q}}_{i},\tilde{\mb{v}}_{i},\tilde{\mbg{\theta}}_{i})\leq 0,\quad \forall i=1,\dots,M,&&\label{inequality_constarints}\\
    &\mathcal{H}(\tilde{\mb{p}}_{i},\tilde{\mb{q}}_{i},\tilde{\mb{v}}_{i},\tilde{\mbg{\theta}}_{i})=0,\quad\forall i=1,\dots,M. &&\label{equality_constarints}
\end{align}
\end{subequations}
where the first condition \eqref{statistical_distance} requires minimizing the statistical distance between the real $p_{\text{real}}$ and synthetic $p_{\text{syn}}$ probability distributions, measured by $\textbf{dist}(\cdot||\cdot)$ (e.g., Wasserstein distance). The second condition \eqref{inequality_constarints} requires the synthetic records to satisfy the inequality constraints  $\mathcal{G}$, including the injection and voltage limits. The last condition \eqref{equality_constarints} requires satisfaction of the power flow equality constraints $\mathcal{H}$. 

Figure~\ref{fig:problem_setup} gives a high-level view of the problem setup. We first train a diffusion model based on $\mathcal{D}$ to learn the underlying probability distribution $p_{\text{real}}$. Then, we sample from the learned distribution $p_{\text{syn}}$ to build a synthetic dataset $\widetilde{\mathcal{D}}$. To ensure the feasibility of the synthetic samples, we guide the sampling process using the power flow constraints.
\begin{figure}
    \centering
    \hspace*{-0.5cm} 
    \includegraphics[width=1.05\linewidth]{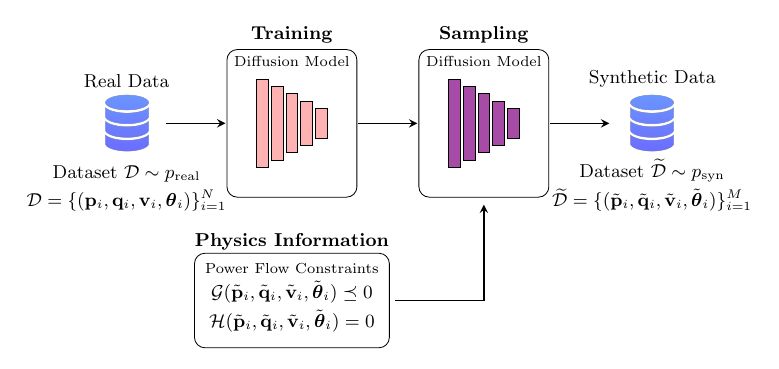}
    \caption{A high-level view of the diffusion model for synthesizing power flow datasets: The training phase (left) uses the actual power flow data $\mathcal{D}$ to learn the real data distribution $p_{\text{real}}$ using a neural network; the sampling phase (right) uses the trained neural network to generate synthetic power flow samples $\widetilde{\mathcal{D}}$; the integration of the power flow constraints (bottom) within the sampling phase ensures that generated samples are physically meaningful.}
    \label{fig:problem_setup}
\end{figure}

\section{Preliminaries} \label{Preliminaries}

This section presents preliminaries on diffusion models and power flow modeling; readers familiar with both topics are invited to proceed to the next section.

\subsection{Diffusion Models}

Diffusion models are generative models that synthesize new data through a two-stage process: forward and reverse. Consider $\mb{x}_0=(\mb{p},\mb{q},\mb{v},\mbg{\theta})$ as a real power flow data point from the underlying distribution of the real data $p_{\text{real}}=q_{0}$. The forward process is a Markov chain that incrementally adds Gaussian noise to a real data point $\mathbf{x}_0\sim q_0$ and transforms it into pure Gaussian noise through a fixed sequence of steps. For each time step $\{t\}_{t=1}^{T}$, the diffusion transition kernel is 
\begin{equation}
q(\mathbf{x}_t \mid \mathbf{x}_{t-1}) = \mathcal{N}(\mathbf{x}_t; \sqrt{1 - \beta_t}\,\mathbf{x}_{t-1}, \beta_t \mathbf{I}),
\label{transition_kernel}
\end{equation}
where \( \beta_t \) is a small positive constant controlling the amount of noise added at each step \cite{ho2020denoising}. From (\ref{transition_kernel}), we directly obtain $\mathbf{x}_t$ from $\mathbf{x}_0$ using
\begin{equation}
    \mathbf{x}_t=\sqrt{\bar{\alpha}_t}\mathbf{x}_0+\sqrt{1 - \bar{\alpha}_t}\mbg{\epsilon},
\end{equation}
where $\alpha_t = 1 - \beta_t$ and $\bar{\alpha}_t = \prod_{s=1}^t \alpha_s$.

The reverse process aims to recover the underlying data distribution $q_0$ from the tractable noise distribution $q_T$. It is modeled as a parameterized Markov chain:
\begin{equation}
p_\theta(\mathbf{x}_{t-1} \mid \mathbf{x}_t) = \mathcal{N}(\mathbf{x}_{t-1}; \mbg{\mu}_\theta(\mathbf{x}_t, t), \mbg{\Sigma}_\theta(\mathbf{x}_t, t)),
\label{reverse_process}
\end{equation}
with the mean $\mbg{\mu}_\theta(\cdot, t)$ and covariance $\mbg{\Sigma}_\theta(\cdot, t)$ functions learned using neural networks parametrized by $\theta$  \cite{ho2020denoising}. 

To train the neural network, we select the loss function as a mean-squared error between the actual noise $\mbg{\epsilon}$ added during the forward process and the noise $\mbg{\epsilon}_\theta(\cdot,t)$ predicted by the neural network:
\begin{equation}
\mathcal{L}_{\text{diff}} = \mathbb{E}_{\mb{x}_0, \epsilon, t} \left\|\mbg{\epsilon} - \mbg{\epsilon}_\theta(\sqrt{\bar{\alpha}_t}\,\mb{x}_0 + \sqrt{1 - \bar{\alpha}_t}\,\mbg{\epsilon}, t) \right\|^2,
\label{loss_function}
\end{equation}
with \( \mbg{\epsilon} \sim \mathcal{N}(0, \mathbf{I}) \) and \( \bar{\alpha}_t \) as defined before. Algorithm~\ref{alg1} summarizes the implementation of the training process \cite{ho2020denoising}.

Once the neural network is trained, new data points are generated using the predicted clean sample $\mb{\hat{x}}_0$ at each step $t$ via Tweedie’s formula:
\begin{equation}
\hat{\mathbf{x}}_0(\mathbf{x}_t, t) = \frac{1}{\sqrt{\bar{\alpha}_t}} \left( \mathbf{x}_t - \sqrt{1 - \bar{\alpha}_t} \, \mbg{\epsilon}_\theta(\mathbf{x}_t, t) \right),
\label{sampling_1}
\end{equation}
and then
\begin{equation} 
\mathbf{x}_{t-1}=\frac{\sqrt{\alpha_t}(1 - \bar{\alpha}_{t-1})}{1 - \bar{\alpha}_t} \mathbf{x}_t + \frac{\sqrt{\bar{\alpha}_{t-1}}\beta_t}{1 - \bar{\alpha}_t} \hat{\mathbf{x}}_0 + {\sigma}_t \mb{z}, \label{sampling_2}
\end{equation}
where $\mb{z}\sim\mathcal{N}(0, \mathbf{I})$, $\sigma_{t}=\beta_{t}\left(1 - \bar{\alpha}_{t-1}\right)/\left(1 - \bar{\alpha}_{t}\right)$, and $t$ ranges from $T$ (starting with the pure Gaussian noise) to 1 (generated sample). Algorithm~\ref{alg2} summarizes the implementation of the sampling process \cite{ho2020denoising}, which returns the synthetic sample $\tilde{\mathbf{x}}_0$ statistically consistent with the original sample $\mb{x}_{0}$. 
\begin{algorithm}[tb]
\caption{:~Training the diffusion model}
\label{alg1}
\textbf{Inputs}: initialized neural network $\mbg{\epsilon}_{\theta}$, noise schedule $\{\alpha_{t}\}_{t=1}^{T}$, dataset of $\mathbf{x}_0$'s sampled from $q_0$\\
\textbf{Outputs}: trained neural network $\mbg{\epsilon}_{\theta}$
\begin{algorithmic}[1]
\Repeat

    \State $\mathbf{x}_0 \sim q_0(\mathbf{x}_0)$
    \State $t \sim \text{Uniform}(\{1, \dots, T\})$
    \State ${\epsilon} \sim \mathcal{N}(0, \mathbf{I})$
    \State Take gradient descent step on
    \[
    \nabla_\theta \left\| {\mbg{\epsilon}} - \mbg{\epsilon}_\theta\left( \sqrt{\bar{\alpha}_t} \mathbf{x}_0 + \sqrt{1 - \bar{\alpha}_t} {\mbg{\epsilon}}, t \right) \right\|^2
    \]
\Until{converged}
\end{algorithmic}
\end{algorithm}
\vspace{0.5em}
\begin{algorithm}[tb]
\caption{:~Sampling new data points}
\label{alg2}
\textbf{Inputs}: trained neural network $\mbg{\epsilon}_{\theta}$, noise schedule $\{\alpha_{t}\}_{t=1}^{T}$, noise scale $\sigma_t$\\
\textbf{Outputs}: new data point $\tilde{\mathbf{x}}_0$
\begin{algorithmic}[1]
\State $\mathbf{x}_T \sim \mathcal{N}(0, \mathbf{I})$
\For{$t = T, \dots, 1$}
    \State $\hat{\mathbf{x}}_0 \gets \frac{1}{\sqrt{\bar{\alpha}_t}} \left( \mathbf{x}_t - \sqrt{1 - \bar{\alpha}_t} \mbg{\epsilon}_\theta(\mathbf{x}_t, t) \right)$
    \State $\mathbf{z} \sim \mathcal{N}(0, \mathbf{I})$ if $t > 1$, else $\mathbf{z} = 0$
    \State $\mathbf{x}_{t-1}\gets\frac{\sqrt{\alpha_t}(1 - \bar{\alpha}_{t-1})}{1 - \bar{\alpha}_t} \mathbf{x}_t + \frac{\sqrt{\bar{\alpha}_{t-1}} \beta_t}{1 - \bar{\alpha}_t} \hat{\mathbf{x}}_0 + {\sigma}_t \mb{z}$
\EndFor
\State \textbf{return} $\tilde{\mathbf{x}}_0$
\end{algorithmic}
\end{algorithm}

\subsection{Power Flow Constraints}
Figure \ref{network_model} illustrates the grid topology and notation used throughout this section. Let $\mathcal{B}=\{1,\cdots,B\}$ denote the set of buses and $\mathcal{L}=\{1,\cdots,L\}$ denote the set of transmission lines in a power grid. Moreover, let elements of power injection vectors $\mb{p}$ and $\mb{q}$ be indexed as $p_{b}$ and $q_{b}$, and let elements of voltage vectors $\mb{v}$ and $\mbg{\theta}$ be indexed as $v_b$ and $\theta_b$, $\forall b \in \mathcal{B}$. In the interest of presentation, we omit shunt admittances in the formulation, though they are included in our numerical results.
\begin{figure}
    \centering
\resizebox{0.5\textwidth}{!}{
   \begin{tikzpicture}
    \draw[black,ultra thick] (0, 0) node[left]{{\scriptsize$v_i\angle\theta_i$}} -- (1, 0) ;
     \draw[black,thick](0.75,0)--(0.75,-0.25);\draw[black,thick](0.75,-0.375) circle (0.125);\draw[black,->,>=stealth] (1.05, -0.45) -- node[pos=0.25,right] {\scriptsize$p_i$} (1.05,-0.05);
    \draw[black,->,>=stealth,thick](0.25,0)--(0.25,-0.5);
    \draw[black,thick] (0.25, 0)  -- (0.25, 0.5) -- node[midway,rotate=45,draw=black,fill=white,minimum width=7.5mm,minimum height=2mm,anchor=center]{} (6*0.25, 7*0.25);\draw[black,thick,dashed](6*0.25, 7*0.25)--(7*0.25, 8*0.25);
    \draw[black,thick] (0.75, 0)  -- (0.75, 0.5) -- node[draw=black,fill=white,minimum width=7.5mm,minimum height=2mm]{} node[below,yshift=-0.15cm] {\scriptsize$g_l,b_l$} 
    (5.25, 0.5)  -- (5.25, 0);
    \draw[black,->,>=stealth] (1, 0.65) -- node[pos=0.5,above] {\scriptsize$f_{l,i\rightarrow j}^{p}$} (1.75,0.65);
    \draw[black,<-,>=stealth] (4.25, 0.65) -- node[pos=0.5,above] {\scriptsize$f_{l,j\rightarrow i}^{p}$} (5,0.65);
    \draw[black,ultra thick] (5, 0)  -- (6, 0) node[right]{\scriptsize$v_j\angle\theta_j$};
    \draw[black,thick](5.25,0)--(5.25,-0.25);\draw[black,thick](5.25,-0.375) circle (0.125);\draw[black,->,>=stealth] (4.95, -0.45) -- node[pos=0.25,left] {\scriptsize$p_j$} (4.95,-0.05);
    \draw[black,->,>=stealth,thick](5.75,0)--(5.75,-0.5);
    \draw[black,thick] (5.75, 0)  -- (5.75, 0.5) -- node[midway,rotate=-45,draw=black,fill=white,minimum width=7.5mm,minimum height=2mm,anchor=center]{} (4.5, 7*0.25);\draw[black,thick,dashed](4.25, 8*0.25)--(4.5, 7*0.25);
    \end{tikzpicture}
}
    \caption{Schematic diagram of the power grid.}
    \label{network_model}
\end{figure}
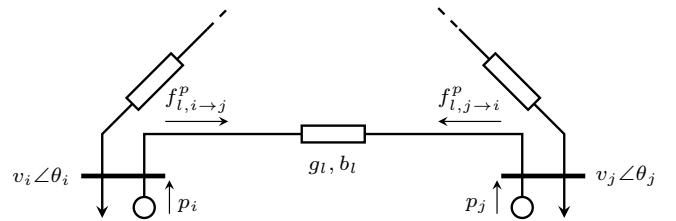
\subsubsection{Power Flow Equality Constraints}
 For each bus $b\in\mathcal{B}$, power flow equality constraints can be represented as
\begin{subequations}
\label{detACOPF1}
\begin{align}    
&p_{b}-\sum_{l \in \mathcal{L}: i=b} f^p_{l, i\to j}-\sum_{l \in \mathcal{L}: j=b} f^p_{l, j\to i}=0, &&\label{ac_powerbal}  \\
&q_{b}- \sum_{l \in \mathcal{L}: i=b} f^q_{l, i\to j} - \sum_{l \in \mathcal{L}: j=b} f^q_{l, j\to i}=0, &&\label{reac_powerbal}  
\end{align}
\end{subequations}
where constraints \eqref{ac_powerbal} and \eqref{reac_powerbal} enforce the active and reactive nodal power balance.
The explicit expression for active power flows $f^{p}_{l,i\to j}$ and reactive power flows $f^{q}_{l,i\to j}$ on each transmission line $l \in \mathcal{L}$ from node $i$ to node $j$ are given by:
\begin{subequations}
\label{ac_pf_eq1}
\begin{align}
f^p_{l,i\to j} =  v_i v_j \big[g_{l} \cos (\theta_i - \theta_j) + b_{l}  \sin (\theta_i - \theta_j)\big],    \label{ac_pf_eq1p}\\
f^q_{l,i\to j}= v_i v_j \big[g_{l} \sin(\theta_i - \theta_j) - b_{l} \cos (\theta_i - \theta_j)\big], \label{ac_pf_eq1q}
\end{align}
\end{subequations}
where $g_{l}=G_{ij}$ and $b_{l}=B_{ij}$ are the real and imaginary parts of the grid admittance matrix $Y=G+jB$. Note that due to line power losses, $f^{p}_{l,i\to j} \neq f^{p}_{l,j\to i}$ and $f^{q}_{l,i\to j} \neq f^{q}_{l,j\to i}$ \cite{molzahn2019survey}. 
\subsubsection{Power Flow Inequality Constraints}
Power flow inequality constraints can be represented as follows:
\begin{subequations}
\label{detACOPF3}
\begin{align}    
&p_{b}^{min} \leq p_{b} \leq p_{b}^{max},\quad\forall {b\in\mathcal{B}}, \label{ac_p}\\
&q_{b}^{min} \leq q_{b} \leq q_{b}^{max},\quad\forall {b\in\mathcal{B}}, \label{ac_q}\\
&v_{b}^{min} \leq v_{b} \leq v_{b}^{max},\quad\forall {b\in\mathcal{B}},\label{ac_v}\\
&(f^p_{l,i\to j})^2+ (f^q_{l,i\to j})^2 \leq(s_{l}^{max})^2,\quad\forall {l\in\mathcal{L}}.\label{ac_s} 
\end{align}
\end{subequations}
where constraints \eqref{ac_p} and \eqref{ac_q} impose limits on the active and reactive power injections, and constraints ~\eqref{ac_v} and \eqref{ac_s} do the same for nodal voltages and apparent power flows.

\section{Diffusion Guidance \\ based on Power Flow Constraints}\label{gradient_guidance}

In theory, a diffusion model trained on feasible power flow data should satisfy constraints \eqref{detACOPF1}--\eqref{detACOPF3}, as they are implicitly encoded in the training dataset. However, in practice, the training and sampling errors may lead to a different outcome \cite{feng2024neural,daras2023consistent}. Although these errors enables the generative power of diffusion models to synthesize new yet statistically consistent samples, the generated power flow samples may not be feasible. In this section, we propose a guidance term for diffusion sampling that preserves the statistical properties of the learned distribution while steering the sampling trajectory toward physically meaningful power flow samples. 

\begin{figure*}[t]
    \centering
    \includegraphics[width=\textwidth]{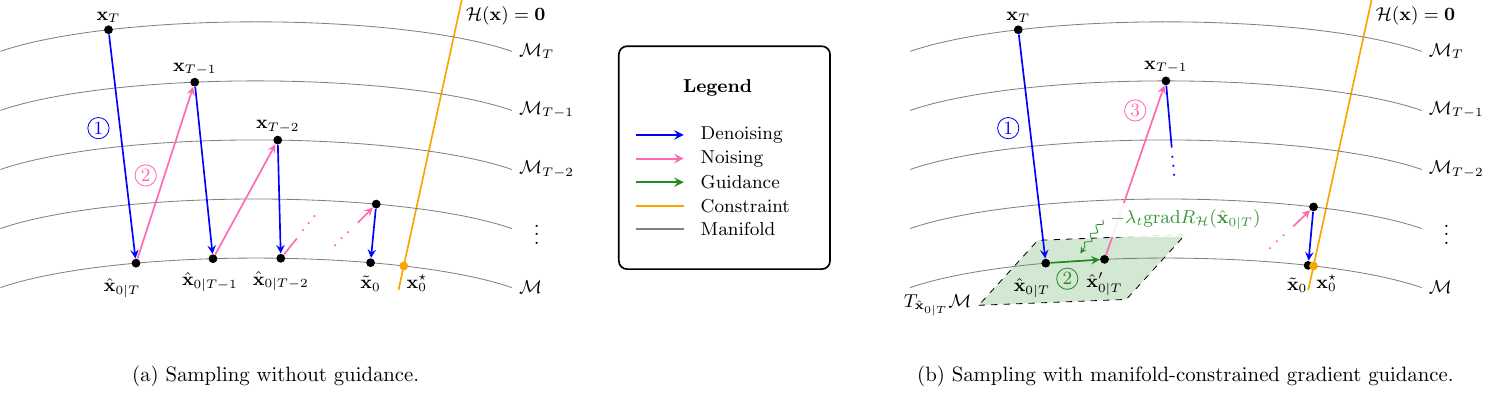}
    \caption{Schematic overview of the geometry of sampling (a) without guidance and (b) with manifold-constrained gradient guidance. In sampling without guidance (a), at each step $t$, we have a 2-stage reverse diffusion step: (1) we do a denoising step based on $\mathbf{x}_t$ and estimate the clean data $\hat{\mathbf{x}}_{0|t}$, and (2) by adding noise with respect to the corresponding noise schedule, we obtain $\mathbf{x}_{t-1}$. In sampling with guidance (b), we have a 3-stage reverse diffusion step: (1) we do a denoising step based on $\mathbf{x}_t$ and estimate the clean data $\hat{\mathbf{x}}_{0|t}$, (2) we add the guidance term based on the gradient of the constraints residual function $R_{\mathcal{H}}(\hat{\mathbf{x}}_{0|t})$ and obtain $\hat{\mathbf{x}}'_{0|t}$ , and (3) by adding noise with respect to the corresponding noise schedule, we obtain $\mathbf{x}_{t-1}$.}
    \label{fig:sampling_geometries}

\end{figure*}

Figure~\ref{fig:sampling_geometries}a illustrates the geometry of standard sampling in a diffusion model. This geometry is characterized by a sequence of manifolds $\{\mathcal{M}_i\}_{i=0}^{T}$. At the bottom, there exists a clean data manifold $\mathcal{M}=\mathcal{M}_0$ surrounded by noisier manifolds according to the noise schedule, where the noisy data resides. Furthermore, let $\mathcal{H}(\mb{x})=0$ represent the power flow equations \eqref{detACOPF1}, where their intersection with the clean data manifold $\mathcal{M}$ is the ideal sample $\mb{x}^{\star}_0$. Accordingly, reverse diffusion steps can be characterized as mere transitions from manifold $\mathcal{M}_i$ to $\mathcal{M}_{i-1}$. At each step $t$, we have a 2-stage reverse diffusion step. First, we do a denoising step based on $\mathbf{x}_t$ and estimate the clean data $\hat{\mathbf{x}}_{0|t}$. Due to the geometric interpretation of diffusion models, a single denoising step at $t$ from manifold $\mathcal{M}_t$ can be viewed as an orthogonal projection onto the clean data manifold $\mathcal{M}$ \cite{chung2022improving}. Then, by adding noise with respect to the corresponding noise schedule, we obtain $\mathbf{x}_{t-1}$. As shown in Fig.~\ref{fig:sampling_geometries}a, the standard sampling process is oblivious to the power flow constraints $\mathcal{H}(\mb{x})=0$. That is, no information about these constraints is incorporated into the sampling process.

To address this issue, we propose to incorporate a guidance term during the sampling process to encourage constraint satisfaction. The proposed guidance term, inspired by manifold-constrained gradients, incorporates the constraint information. Specifically, we define a data consistency loss function as a residual of power flow constraints $\mathcal{H}(\mb{x})$:
\begin{equation}
    R_{\mathcal{H}}(\mathbf{x}) = \|\mathcal{H}(\mathbf{x})\|_2^2,
\end{equation}
and aim to minimize this loss over the clean data manifold $\mathcal{M}$, which is implicitly learned by the diffusion model:
\begin{equation}
    \min_{\mathbf{x} \in \mathcal{M}} R_{\mathcal{H}}(\mathbf{x}).
    \label{mcg}
\end{equation}
To guide sampling trajectories at each denoising step, we apply a Riemannian gradient descent with respect to (\ref{mcg}):
\begin{equation}\label{gradient_step_Riemannian}
    \hat{\mathbf{x}}_{0|t}^{\prime} = \hat{\mathbf{x}}_{0|t} - \lambda_t~\text{grad}~R_{\mathcal{H}}(\hat{\mathbf{x}}_{0|t}),
\end{equation}
where $\text{grad}~R_{\mathcal{H}}(\hat{\mathbf{x}}_{0|t})$ denotes the Riemannian gradient of $R_{\mathcal{H}}$ at $\hat{\mathbf{x}}_{0|t}$, defined as the projection of the Euclidean gradient onto the tangent space of the manifold $\mathcal{M}$ at $\hat{\mathbf{x}}_{0|t}$, i.e., $T_{\hat{\mathbf{x}}_{0|t}}\mathcal{M}$  \cite{boumal2023introduction}:
\begin{equation}\label{eq:projection}
    \text{grad}~R_{\mathcal{H}}(\hat{\mathbf{x}}_{0|t}) = \mathcal{P}_{T_{\hat{\mathbf{x}}_{0|t}}\mathcal{M}} \left(\nabla_{{\mathbf{x}}_{t}} R_{\mathcal{H}}(\hat{\mathbf{x}}_{0|t})\right).
\end{equation}

However, since the clean data manifold $\mathcal{M}$ is not known explicitly, this makes it intractable to compute the projection operator $\mathcal{P}_{T_{\hat{\mathbf{x}}_{0|t}}\mathcal{M}}$ directly. Fortunately, under a local linearity assumption \cite{chung2022improving}, it can be shown that the Euclidean gradient of $R_{\mathcal{H}}$ at $\hat{\mathbf{x}}_{0|t}$ is already aligned with the tangent space of $\mathcal{M}$ at $\hat{\mathbf{x}}_{0|t}$, making the projection step unnecessary. Leveraging the main theorem from \cite{chung2022improving} on manifold-constrained gradients, we formalize this insight in the following theorem.
\begin{theorem}\label{theorem_1}
Let $\mathcal{M}$ denote the clean data manifold, and assume that in a local neighborhood of $\hat{\mathbf{x}}_{0|t}$, $\mathcal{M}$ is well approximated by an affine subspace. Then, the gradient of the residual function $R_{\mathcal{H}}(\hat{\mathbf{x}}_{0|t})$ is tangential to $\mathcal{M}$, i.e.,
\begin{equation}
    \mathcal{P}_{T_{\hat{\mathbf{x}}_{0|t}}\mathcal{M}}\left(\nabla_{{\mathbf{x}}_{t}} R_{\mathcal{H}}(\hat{\mathbf{x}}_{0|t})\right)=\nabla_{{\mathbf{x}}_{t}} R_{\mathcal{H}}(\hat{\mathbf{x}}_{0|t}).
\end{equation}
\end{theorem}
\begin{proof}[Proof]
Let $\mb{x}_t$ denote a noisy sample at diffusion step $t$, and let $Q:\mathbb{R}^{4B} \rightarrow \mathbb{R}^{4B}$ denote the function that maps $\mb{x}_t$  to its corresponding clean estimate $\hat{\mathbf{x}}_{0|t}$:
\begin{equation}
\hat{\mathbf{x}}_{0|t} =Q(\mb{x}_t)= \frac{1}{\sqrt{\bar{\alpha}_t}} \left( \mathbf{x}_t - \sqrt{1 - \bar{\alpha}_t}~\epsilon_\theta(\mathbf{x}_t, t) \right).
\end{equation}
Recall the data consistency loss function
\begin{equation}\label{eq:data_consistency_loss}
  R_{\mathcal{H}}(Q(\mb{x}_t)) = \left\| \mathcal{H}(Q(\mb{x}_t)) \right\|_2^2,  
\end{equation}
which evaluates the violation of the constraint function $\mathcal{H}$ on the clean estimate \( \hat{\mathbf{x}}_{0|t} = Q(\mb{x}_t) \).

We are interested in computing the gradient of (\ref{eq:data_consistency_loss}) with respect to the noisy input $\mb{x}_t$, i.e., $\nabla_{\mb{x}_t} R_{\mathcal{H}}(Q(\mb{x}_t))$. Applying the chain rule yields:
\begin{align}\label{eq:chain_rule}
\nabla_{\mathbf{x}_t} R_{\mathcal{H}}(Q(\mathbf{x}_t)) 
&= \nabla_{\mathbf{x}_t} \left\| \mathcal{H}(Q(\mathbf{x}_t)) \right\|_2^2 \notag \\
&= \nabla_{\mathbf{x}_t} \left( \mathcal{H}(Q(\mathbf{x}_t))^\top \mathcal{H}(Q(\mathbf{x}_t)) \right) \notag \\
&= 2 \left( \nabla_{\mathbf{x}_t} \mathcal{H}(Q(\mathbf{x}_t)) \right)^\top \mathcal{H}(Q(\mathbf{x}_t)) \notag \\
&= 2 \left( J_{\mathcal{H}}(Q(\mathbf{x}_t))~J_Q(\mathbf{x}_t) \right)^\top \mathcal{H}(Q(\mathbf{x}_t)) \notag \\
&= 2 J_Q(\mathbf{x}_t)^\top J_{\mathcal{H}}(Q(\mathbf{x}_t))^\top \mathcal{H}(Q(\mathbf{x}_t)),
\end{align}
where $J_Q(\mb{x}_t) $ is the Jacobian of the map $Q$ and $J_{\mathcal{H}}(Q(\mb{x}_t))$ is the Jacobian of the constraint function $\mathcal{H}$ evaluated at the clean data estimate $\hat{\mathbf{x}}_{0|t}=Q(\mathbf{x}_t)$.

Now, according to Proposition~2 in \cite{chung2022improving}, the map $Q$ behaves locally as an orthogonal projection onto the clean data manifold $\mathcal{M}$:
\begin{equation}
Q(\mb{x}_t) \in \mathcal{M},
\end{equation}
\begin{equation}
J_Q(\mb{x}_t) = J_Q(\mb{x}_t)^\top = J_Q(\mb{x}_t)^2,
\end{equation}
which implies that $J_Q(\mb{x}_t)$ is an orthogonal projection that projects onto the tangent space $T_{Q(\mb{x}_t)}\mathcal{M}$ at $Q(\mathbf{x}_t)$. As a result, the gradient
\begin{equation}
\nabla_{\mb{x}_t} R_{\mathcal{H}}(Q(\mb{x}_t)) = J_Q(\mb{x}_t)^\top v,
\end{equation}
where $v=2 J_{\mathcal{H}}(Q(\mb{x}_t))^\top \mathcal{H}(Q(\mb{x}_t))$ due to (\ref{eq:chain_rule}), is already in the tangent space \( T_{Q(\mb{x}_t)}\mathcal{M} \). Therefore, projecting it onto the tangent space does not change it:
\begin{equation}
\mathcal{P}_{T_{Q(\mb{x}_t)} \mathcal{M}}\left( \nabla_{\mb{x}_t} R_{\mathcal{H}}(Q(\mb{x}_t)) \right) 
= \nabla_{\mb{x}_t} R_{\mathcal{H}}(Q(\mb{x}_t)).
\end{equation}

Hence, the gradient of the constraint residual function $R_{\mathcal{H}}(Q(\mb{x}_t))$ with respect to the noisy sample $\mb{x}_t$ lies in the tangent space of the clean data manifold $\mathcal{M}$ at $\hat{\mathbf{x}}_{0|t}$; thus, no projection is needed.
\end{proof}

Substituting the result from Theorem~\ref{theorem_1} in (\ref{gradient_step_Riemannian}) yields the following practical correction rule for gradient guidance:
\begin{equation}
\label{gradient_step_Euclidean}
    \hat{\mathbf{x}}_{0|t}^{\prime} = \hat{\mathbf{x}}_{0|t} - \lambda_t~\nabla_{{\mathbf{x}}_{t}} R_{\mathcal{H}}(\hat{\mathbf{x}}_{0|t}),
\end{equation}
where $\lambda_t$ is a hyperparameter controlling the strength of the guidance at step $t$. Although a moderate $\lambda_t$ encourages the generation of feasible samples, excessively large values can distort the sampling trajectory, pushing samples off the data manifold or even causing instability. This occurs because large values of $\lambda_t$ violate the affine subspace assumption in Theorem~\ref{theorem_1}, thereby undermining the validity of the guidance direction. Conversely, small $\lambda_t$ results in samples that violate the constraints. Hence, $\lambda_t$ should be carefully tuned in practice to balance constraint satisfaction and statistical representation.

Figure~\ref{fig:sampling_geometries}b illustrates the geometry of sampling with the manifold-constrained gradient guidance. Unlike standard sampling, an additional step is incorporated based on the gradient of the constraint residual function $R_{\mathcal{H}}(\hat{\mathbf{x}}_{0|t})$. The guidance term steers the sampling trajectory toward the intersection $\mathbf{x}^{\star}_0$ of the constraint $\mathcal{H}(\mb{x})=0$ and the clean data manifold $\mathcal{M}$. Since the guidance term is tangential to the clean data manifold $\mathcal{M}$, $\hat{\mathbf{x}}'_{0|t}$ remains on the clean data manifold $\mathcal{M}$, ensuring that the final sample is both feasible and statistically representative.

To enforce inequality constraints \eqref{detACOPF3}, we implement a similar approach. Consider the following inequality constraints
\begin{equation}
\mathcal{G}(\mathbf{x})\leq 0,
\end{equation}
for which the residual function $R_\mathcal{G}(\cdot)$ is defined as
\begin{equation}
R_\mathcal{G}(\mathbf{x})=\|\max(\mathcal{G}(\mathbf{x}),0)\|_{2}^{2}.
\end{equation}
The correction rule is similar to \eqref{gradient_step_Euclidean}, with the gradient guidance term defined as
\begin{equation}\label{ineq_guidance}
   -\nabla_{{\mathbf{x}}_{t}} R_{\mathcal{G}}(\hat{\mathbf{x}}_{0|t})=-\nabla_{{\mathbf{x}}_{t}}\|\max(\mathcal{G}(\hat{\mathbf{x}}_{0|t}),0)\|_{2}^{2}.
\end{equation}

The gradient guidance terms are incorporated into the sampling process as shown in Algorithm~\ref{alg3}. The full expressions of the guidance terms, specific to the AC power flow constraints, are provided in the e-companion of this paper \cite{hoseinpour2025constrained}. 
Due to Step 4, the guidance terms modify the sampling path at each reverse diffusion step. We also omit the subscript $t$ from the estimated clean data $\hat{\mathbf{x}}_{0|t}$ and denote it by $\hat{\mathbf{x}}_{0}$.

\begin{algorithm}[tb]
\caption{: Sampling with gradient guidance}
\label{alg3}
\textbf{Inputs}: trained neural network $\mbg{\epsilon}_{\theta}$, noise schedule $\{\alpha_{t}\}_{t=1}^{T}$, noise scale $\sigma_t$, guidance scale $\lambda_t$\\
\textbf{Outputs}: new data point $\tilde{\mathbf{x}}_0$
\begin{algorithmic}[1]
\State $\mathbf{x}_T \sim \mathcal{N}(0,  \mathbf{I}_{4B})$
\For{$t = T-1$ to $0$}
    \State $\hat{\mathbf{x}}_{0} \gets \frac{1}{\sqrt{\bar{\alpha}_t}} \left( \mathbf{x}_t - \sqrt{1 - \bar{\alpha}_t} \mbg{\epsilon}_{\theta}(\mathbf{x}_t, t) \right)$
     \State $\hat{\mathbf{x}}_{0}^{\prime} \gets \hat{\mathbf{x}}_{0} - \lambda_t~\left(\nabla_{{\mathbf{x}}_{t}} R_{\mathcal{H}}(\hat{\mathbf{x}}_{0})+\nabla_{{\mathbf{x}}_{t}} R_{\mathcal{G}}(\hat{\mathbf{x}}_{0})\right)$
    \State $z \sim \mathcal{N}(0,  \mathbf{I}_{4B})$
    \State $\mathbf{x}_{t-1} \gets \frac{\sqrt{\alpha_t}(1 - \bar{\alpha}_{t-1})}{1 - \bar{\alpha}_t} \mathbf{x}_t + \frac{\sqrt{\bar{\alpha}_{t-1}} \beta_t}{1 - \bar{\alpha}_t}\hat{\mathbf{x}}'_{0} + \sigma_t \mb{z}$
\EndFor
\State\textbf{return} $\tilde{\mathbf{x}}_0$
\end{algorithmic}
\end{algorithm}
\section{Practical Implementation via Variable Decoupling and Normalization}\label{implementation_strategies}
We present two practical techniques that leverage domain knowledge in power systems to improve (i) computational efficiency of the proposed constrained diffusion model and  (ii) scale consistency of the gradient guidance during sampling. 
\subsection{Variable Decoupling for Computational Efficiency}
\label{variable_decoupling}
In high-voltage transmission systems, active power injection $\mathbf{p}$ highly correlates with $\mbg{\theta}$ and less so with voltage magnitude $\mathbf{v}$, while reactive power injection $\mathbf{q}$ primarily correlates with $\mathbf{v}$ and weakly correlates with phase angle $\mbg{\theta}$ \cite{portelinha2021fast,monticelli2012state}. This observation underlies the classical fast decoupled power flow method and motivates our variable decoupling strategy. We split the full vector $\mb{x}_0=(\mb{p},\mb{q},\mb{v},\mbg{\theta})\in\mathbb{R}^{4B}$ into two lower-dimensional vectors  $\mb{x}^{(1)}_{0}=(\mathbf{p}, \mbg{\theta})\in\mathbb{R}^{2B}$ and $\mb{x}^{(2)}_{0}=(\mathbf{q}, \mathbf{v})\in\mathbb{R}^{2B}$. The diffusion loss $\mathcal{L}_{\text{diff}}$ in \eqref{loss_function} is thus split between two denoiser neural networks:
\begin{equation}
\mathcal{L}_{\text{diff}}=\mathcal{L}^{(1)}_{\text{diff}}+\mathcal{L}^{(2)}_{\text{diff}},
\label{loss_decoupled}
\end{equation}
where $\mathcal{L}^{(1)}_{\text{diff}}$ and $\mathcal{L}^{(2)}_{\text{diff}}$ correspond to $\mb{x}^{(1)}_{0}$ and $\mb{x}^{(2)}_{0}$, respectively.

Due to (\ref{loss_function}), $\mathcal{L}^{(1)}_{\text{diff}}$ is defined as
\begin{equation}
\mathcal{L}^{(1)}_{\text{diff}}= \mathbb{E}_{\mb{x}^{(1)}_0, \mbg{\epsilon}_{1}, t}  \left\|\mbg{\epsilon}_{1} - \mbg{\epsilon}_{\theta_1}(\sqrt{\bar{\alpha}_t}\,\mathbf{x}^{(1)}_0 + \sqrt{1 - \bar{\alpha}_t}\,\mbg{\epsilon}_1, t) \right\|^2,
\label{loss_function_p&theta}
\end{equation}
where $\mbg{\epsilon}_{1}$ and $\mbg{\epsilon}_{\theta_1}$ are the actual and predicted noise at time step $t$ of the forward process. Similarly, $\mathcal{L}^{(2)}_{\text{diff}}$ is defined as
\begin{equation}
\mathcal{L}^{(2)}_{\text{diff}}= \mathbb{E}_{\mb{x}^{(2)}_0, \mbg{\epsilon}_{2}, t}  \left\|\mbg{\epsilon}_{2} - \mbg{\epsilon}_{\theta_2}(\sqrt{\bar{\alpha}_t}\,\mathbf{x}^{(2)}_0 + \sqrt{1 - \bar{\alpha}_t}\,\mbg{\epsilon}_2, t) \right\|^2,
\label{loss_function_q&v}
\end{equation}
where $\mbg{\epsilon}_{2}$ and $\mbg{\epsilon}_{\theta_2}$ are the actual and predicted noise. Algorithm~\ref{alg4} demonstrates the training of the diffusion model under variable decoupling. Similarly, to generate new data points, Algorithm~\ref{alg2} can be adapted based on the variable decoupling approach, resulting in Algorithm~\ref{alg5}.
\begin{algorithm}[tb]
\caption{:~Training the diffusion model under variable decoupling}
\label{alg4}
\textbf{Inputs}: initialized neural networks $\mbg{\epsilon}_{\theta_1}$ and $\mbg{\epsilon}_{\theta_2}$, noise schedule $\{\alpha_{t}\}_{t=1}^{T}$, dataset of $\mathbf{x}_0$'s sampled from $q_0$\\
\textbf{Outputs}: trained neural networks $\mbg{\epsilon}_{\theta_1}$ and $\mbg{\epsilon}_{\theta_2}$
\begin{algorithmic}[1]
\Repeat

    \State $\mathbf{x}_0 \sim q_0(\mathbf{x}_0)$
    \State $\mathbf{x}^{(1)}_{0},\mathbf{x}^{(2)}_{0} \gets \mathbf{x}_{0}$ \hspace{1em} $\triangleright$ split vector

    \State $t \sim \text{Uniform}(\{1, \dots, T\})$
    \State $\mbg{\epsilon}_1,\mbg{\epsilon}_2 \sim \mathcal{N}(0, \mathbf{I}_{2B})$
    \State Take gradient descent step on
    \begin{equation*}
        \begin{split}
             &\mathbb{E}_{\mb{x}^{(1)}_0, \mbg{\epsilon}_{1}, t}  \left\|\mbg{\epsilon}_{1} -\mbg{\epsilon}_{\theta_1}(\sqrt{\bar{\alpha}_t}\,\mathbf{x}^{(1)}_0 + \sqrt{1 - \bar{\alpha}_t}\,\mbg{\epsilon}_1, t) \right\|^2+\\
             &\mathbb{E}_{\mb{x}^{(2)}_0, \mbg{\epsilon}_{2}, t}  \left\|\mbg{\epsilon}_{2} - \mbg{\epsilon}_{\theta_2}(\sqrt{\bar{\alpha}_t}\,\mathbf{x}^{(2)}_0 + \sqrt{1 - \bar{\alpha}_t}\,\mbg{\epsilon}_2, t) \right\|^2
        \end{split}
    \end{equation*}
   
\Until{converged}
\end{algorithmic}
\end{algorithm}
\begin{algorithm}[tb]
\caption{: Sampling with gradient guidance under variable decoupling}
\label{alg5}
\textbf{Inputs}: trained neural networks $\mbg{\epsilon}_{\theta_1}$ and $\mbg{\epsilon}_{\theta_2}$, noise schedule $\{\alpha_{t}\}_{t=1}^{T}$, noise scale $\sigma_t$, guidance scale $\lambda_t$\\
\textbf{Outputs}: new data point $\tilde{\mathbf{x}}_0$
\begin{algorithmic}[1]
\State $\mathbf{x}_{T}  \sim \mathcal{N}(0, \mathbf{I}_{4B})$
\State $\mathbf{x}^{(1)}_{T},\mathbf{x}^{(2)}_{T} \gets \mathbf{x}_{T}$ \hspace{1em} $\triangleright$ split vector

\For{$t = T-1$ to $0$}
    \vspace{.1cm}
    \State $\hat{\mathbf{x}}^{(1)}_{0} \gets \frac{1}{\sqrt{\bar{\alpha}_t}} \left( \mathbf{x}^{(1)}_t - \sqrt{1 - \bar{\alpha}_t}\mbg{\epsilon}_{\theta_1} (\mathbf{x}^{(1)}_t, t) \right)$
    \vspace{.1cm}
    \State $\hat{\mathbf{x}}^{(2)}_{0} \gets \frac{1}{\sqrt{\bar{\alpha}_t}} \left( \mathbf{x}^{(2)}_t - \sqrt{1 - \bar{\alpha}_t} \mbg{\epsilon}_{\theta_2}(\mathbf{x}^{(2)}_t, t) \right)$
    \State $\hat{\mathbf{x}}_{0} \gets\hat{\mathbf{x}}^{(1)}_{0}\mathbin\Vert\hat{\mathbf{x}}^{(2)}_{0}$ \hspace{1em} $\triangleright$ concatenate vectors

     \State $\hat{\mathbf{x}}_{0}^{\prime} \gets \hat{\mathbf{x}}_{0} - \lambda_t~\left(\nabla_{{\mathbf{x}}_{t}} R_{\mathcal{H}}(\hat{\mathbf{x}}_{0})+\nabla_{{\mathbf{x}}_{t}} R_{\mathcal{G}}(\hat{\mathbf{x}}_{0})\right)$
    \State $\mb{z} \sim \mathcal{N}(0,  \mathbf{I}_{4B})$
    \State $\mathbf{x}_{t-1} \gets \frac{\sqrt{\alpha_t}(1 - \bar{\alpha}_{t-1})}{1 - \bar{\alpha}_t} \mathbf{x}_t + \frac{\sqrt{\bar{\alpha}_{t-1}} \beta_t}{1 - \bar{\alpha}_t}\hat{\mathbf{x}}'_{0} + \sigma_t \mb{z}$
    \State $\mathbf{x}^{(1)}_{t-1},\mathbf{x}^{(2)}_{t-1} \gets \mathbf{x}_{t-1}$ \hspace{1em} $\triangleright$ split vector

\EndFor
\State\textbf{return} $\tilde{\mathbf{x}}_0$
\end{algorithmic}
\end{algorithm}
\subsection{Normalization for Scale-Consistent Gradient Guidance}
Normalization in power systems is traditionally achieved via p.u. transformation \cite{glover2001power}, bringing all variables to a common basis. However, p.u. transformation alone is insufficient for diffusion, as it does not normalize the variables to a unified numerical range. In fact, the power flow variables $\mb{x}_0=(\mb{p},\mb{q},\mb{v},\mbg{\theta})$ in p.u. still have different numerical ranges and scales. Hence, when computing the gradient guidance term, the difference in scales becomes problematic. Specifically, the elements of the resulting guidance vector inherit the magnitudes of the corresponding variables. As a result, a single scalar guidance scale $\lambda$ may have inconsistent effects, hindering both the convergence and constraint satisfaction during sampling. 

To address this issue, we propose a new normalization of the power flow variables to ensure that the guidance vector has a comparable scale across all the variable types. First, we apply the min-max normalization to the real data prior to training the denoiser, ensuring that all the power flow variables are mapped to the range $[-1,1]$. Specifically, for each data point $\mb{x}_0$, its $i$-{th} variable $\mb{x}_{0,i}$ is transformed as
\begin{equation}
{\mb{x}}^{\mathrm{norm}}_{0,i} = 2\frac{\mb{x}_{0,i} -\mb{x}_{0,i}^{\min}}{\mb{x}_{0,i}^{\max} - \mb{x}_{0,i}^{\min}}-1,\quad \forall i=1,\cdots,4B,
\end{equation}
where $\mb{x}_{0,i}^{\min}$ and $\mb{x}_{0,i}^{\max}$ denote the minimum and maximum values of the $i$-{th} variable in the dataset. The denoiser is then trained using this normalized dataset. Consequently, the entire sampling process is carried out in the normalized space.

To compute the gradient guidance term, we first denormalize the current estimate of the clean sample $\hat{\mb{x}}_0$ using the denormalization function $f_{\text{de}}(\cdot)$:
\begin{equation}
f_{\mathrm{de}}(\hat{\mb{x}}_{0,i}^{\mathrm{norm}})=\left(\frac{\hat{\mb{x}}_{0,i}^{\mathrm{norm}}+1}{2}\right)\left({\mb{x}}_{0,i}^{\mathrm{max}}-{\mb{x}}_{0,i}^{\mathrm{min}}\right)+{\mb{x}}_{0,i}^{\mathrm{min}}.
\end{equation}
The residuals are then evaluated based on actual values. Yet, the derivative in the gradient guidance term is taken with respect to normalized values. By chain rule, we thus have
\begin{equation}
\nabla_{\hat{\mathbf{x}}_0^{\mathrm{norm}}} R_{\mathcal{H}}(\hat{\mathbf{x}}_0^{\mathrm{norm}}) 
= \frac{\partial R_{\mathcal{H}}(f_{\mathrm{de}}(\hat{\mathbf{x}}_0^{\mathrm{norm}}))}{\partial f_{\mathrm{de}}(\hat{\mathbf{x}}_0^{\mathrm{norm}})} 
 \frac{\partial f_{\mathrm{de}}(\hat{\mathbf{x}}_0^{\mathrm{norm}})}{\partial \hat{\mathbf{x}}_0^{\mathrm{norm}}}.
\end{equation}
This approach ensures numerical stability during sampling, while enabling scale-consistent guidance.
\section{Numerical Results}
\label{results}
We evaluate the performance of the proposed constrained diffusion model for synthesizing power flow datasets. We run experiments on three benchmark systems: PJM 5-bus system, IEEE 24-bus system, and IEEE 118-bus system \cite{babaeinejad2019power}. The effectiveness of our approach is evaluated through three analyses: (i) comparing the statistical properties of the synthesized data with those of the ground truth data (Sec.~\ref{results:statistical_similarity}), (ii) analyzing constraint satisfaction (Sec.~\ref{results:constraint_satisfaction}), and (iii) evaluating the utility of the synthesized data in a ML application (Sec.~\ref{results:downstream_task}).

\subsection{Experimental Setup}

Since real-world power flow datasets are not publicly available, we generate the ground truth dataset by applying random perturbations around the nominal load ${s}^{\text{nom}}_b=({{p}}^{\text{nom}}_{d,b},{{q}}^{\text{nom}}_{d,b})$ at each bus $b\in\mathcal{B}$ of the benchmark systems. We uniformly sample active and reactive power demands at each bus ${s}_{b}=({p}_{d,b},{q}_{d,b})\sim \text{Uniform}~(0.8~{s}^{\text{nom}}_b,{s}^{\text{nom}}_b)$, and then solve the AC-OPF problem to extract the feasible solutions $\mb{p},\mb{q},\mb{v},\mbg{\theta}$. By stacking them together into a single data point $(\mb{p}_{i},\mb{q}_{i},\mb{v}_{i},\mbg{\theta}_{i})$, we obtain the ground truth dataset $\mathcal{D}=\{(\mb{p}_{i},\mb{q}_{i},\mb{v}_{i},\mbg{\theta}_{i})\}_{i=1}^{N}$. 

\subsection{ Statistical Similarity}\label{results:statistical_similarity}
The histograms of the ground truth and synthetic power flow variables for the PJM 5-bus system are given in Fig.~\ref{hist_real_vs_syn}. The synthesized variables capture the underlying distribution of the ground truth data. For each class of variables, the synthetic data not only aligns well with the support of the ground truth data but also successfully captures the modes. The proposed diffusion model also ensures similarity of the joint probability distributions, as shown in Fig.~\ref{joint_PQ} and Fig.~\ref{joint_VTHETA} depicting the joint distributions of the active and reactive power injections, and voltage magnitude and phase angel, respectively. If the ground truth data exhibits a multi-modal structure, the synthetic data points successfully capture all modes. Moreover, in regions where the ground truth data is denser, the synthetic data points are also more concentrated, whereas in regions where the density of the ground truth data points decreases, the synthetic data points become more sparse. Another important property is the coverage capability, as shown in Fig.~\ref{joint_PQ} and Fig.~\ref{joint_VTHETA}, where the synthesized data points closely span the entire domain of the ground truth data. 

\begin{figure*}
    \centering
    \subfloat{
        \includegraphics[width=0.95\linewidth]{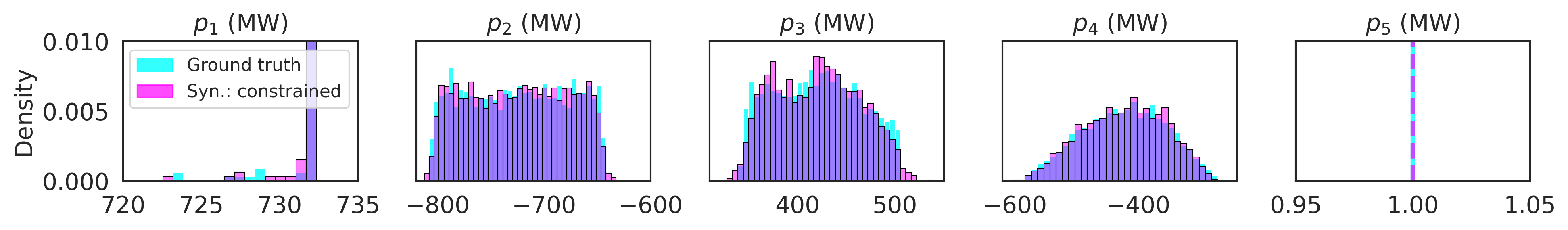}\label{real_vs_synthetic_Pnet}
    }\vspace{-5mm}
    
    \subfloat{
        \includegraphics[width=0.95\linewidth]{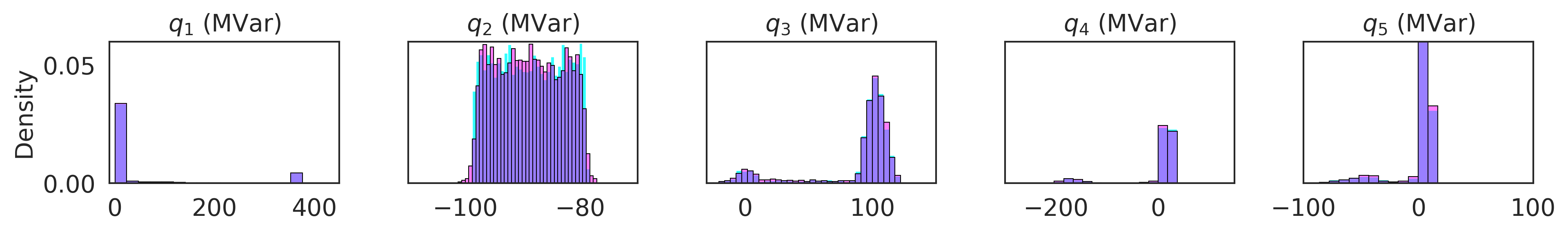}\label{real_vs_synthetic_Qnet}
    }\vspace{-5mm}
    
    \subfloat{
        \includegraphics[width=0.95\linewidth]{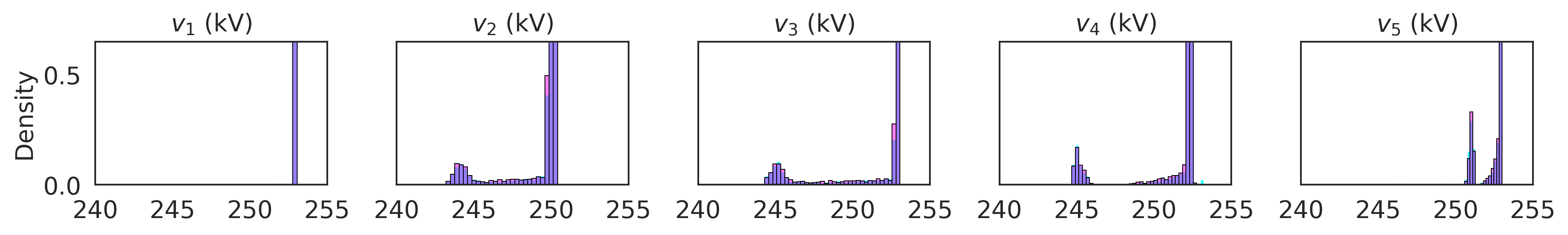}
    }\vspace{-5mm}
    
    \subfloat{
        \includegraphics[width=0.95\linewidth]{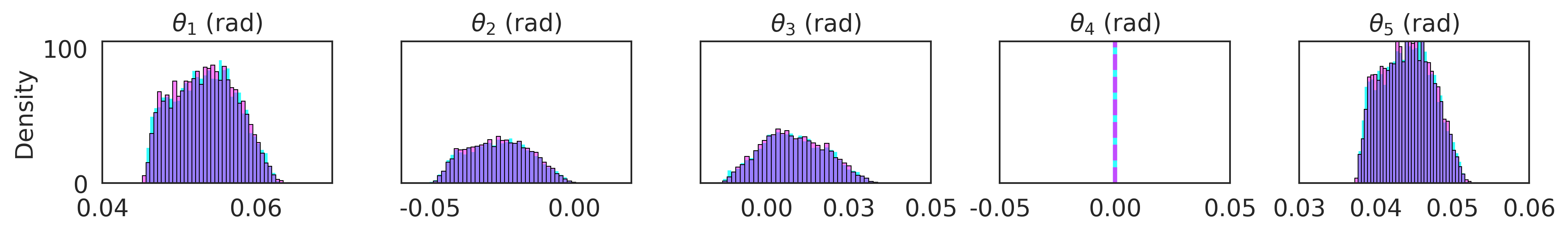}
    }
    \caption{Histograms of the ground truth versus synthetic power flow data points for active power injections (first row), reactive power injections (second row), voltage magnitudes (third row), and phase angles (forth row) at each bus in the PJM 5-bus system.}
    \label{hist_real_vs_syn}
\end{figure*}

\begin{figure*}
    \centering
    \subfloat{
        \includegraphics[width=0.95\linewidth]{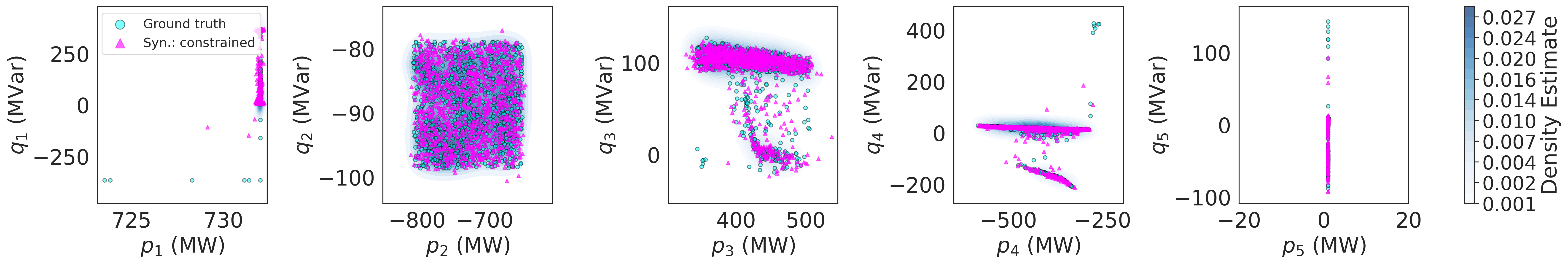}\label{joint_PQ}
    }\vspace{-2mm}
    \subfloat{
        \includegraphics[width=0.95\linewidth]{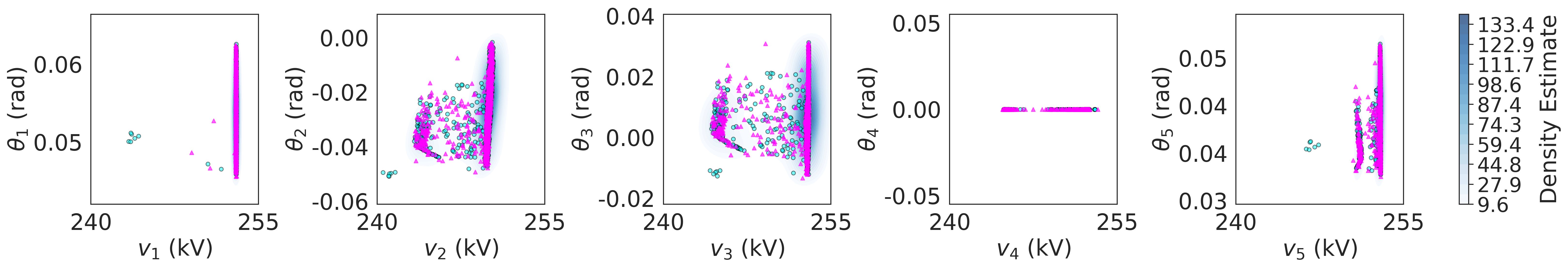}\label{joint_VTHETA}
    }\vspace{0mm}
    
    \caption{2D scatter plots with density estimates of the active and reactive power injection (top row), and voltage magnitude and phase angle (bottom row) at each bus in the PJM 5-bus system, comparing the joint distributions of the ground truth and constrained synthetic datasets. The plots highlight the ability of the diffusion model to replicate the underlying pattern, domain, and multi-modal structure of the ground truth data.} 
\end{figure*}

\begin{table}[t!]
\centering
\caption{Wasserstein distances between the ground truth $\mathcal{D}$ and synthetic data $\widetilde{\mathcal{D}}$}
\label{tab:wasserstein_synthetic_quality}
\renewcommand{\arraystretch}{1.2}
\setlength{\tabcolsep}{8pt}

\begin{tabular}{lccc}
\toprule
Distance between... & 5-Bus & 24-Bus & 118-Bus \\
\midrule
...$\mathcal{D}$ and $\widetilde{\mathcal{D}}$ w/o guidance & 0.442 & 0.607 & 0.622 \\
...$\mathcal{D}$ and $\widetilde{\mathcal{D}}$ w/ guidance & 0.382 & 0.585 & 0.597 \\
\bottomrule
\end{tabular}
\end{table}

To quantify the similarity of the ground truth and synthetic datasets, we use the type-1 Wasserstein distance between these datasets, defined as 
\begin{align}
W_1(\mathcal{D}, \widetilde{\mathcal{D}}) = \min_{\gamma \in \Gamma(\mathcal{D}, \widetilde{\mathcal{D}})} \sum_{i=1}^N \sum_{j=1}^M \gamma_{ij} \| \mb{x}_i - \tilde{\mb{x}}_j \|_2,
\end{align}
where $\Gamma(\mathcal{D}, \widetilde{\mathcal{D}})$ represents the set of all valid ways to assign mass between the ground truth and synthetic samples~\cite{peyré2020computationaloptimaltransport}. The results for synthetic datasets obtained with and without gradient guidance are summarized in Table~\ref{tab:wasserstein_synthetic_quality}. Lower Wasserstein distances indicate closer alignment between synthetic and ground truth distributions. Across all the test systems, the distances remain low, showing that the synthesized power flow data closely mirrors the ground truth data. Enforcing the constraints during sampling further reduces the Wasserstein distance consistently. Thus, we validate that constraint enforcement not only promotes feasibility but also enhances the statistical similarity of the ground truth and synthetic datasets.

\subsection{Constraint Satisfaction}\label{results:constraint_satisfaction}
We evaluate constraint satisfaction of the synthetic datasets generated with and without gradient guidance, specifically focusing on the active and reactive power balance at each bus. Figure~\ref{power_mismatch} presents the histograms of violation magnitudes in the PJM 5-bus system: without the guidance mechanism, a significant portion of the generated samples exhibit non-negligible violations at buses 1, 2, 3, and 5 for the active power balance constraints. With guidance, the vast majority of the samples show near-zero active power mismatch at all buses, indicating strong constraint satisfaction. Similar observations hold for the reactive power balance constraints.

For the IEEE 24-bus system, Table~\ref{tab:mismatch_24bus} reports the mean and variance of the active and reactive power mismatches at each bus. Similar to the PJM 5-bus system, gradient guidance significantly reduces both the mean and variance of constraint violations across most buses. For the IEEE 118-bus system, the results are visualized in Fig.~\ref{fig:pw_mismatch_118} comparing the mean and variance of the power mismatches with and without guidance. Guidance leads to a noticeable improvement for some buses, particularly buses 4 and 5, where the baseline violation is relatively large. For most other buses, the initial violation magnitude is small, which limits the impact of the guidance mechanism. Furthermore, for some buses, the mismatch mean under the constrained sampling is slightly larger. However, this does not contradict the overall effectiveness of the guidance mechanism. We observe that variance plays a more critical role than the mean in determining the quality of constraint satisfaction. That is, a model with a low mismatch mean but high variance may still generate many samples with large constraint violations. The proposed guidance mechanism, instead, ensures that most generated samples remain close to the full satisfaction of the physical constraints.
\begin{figure*}
    \centering
    \subfloat{
        \includegraphics[width=0.95\linewidth]{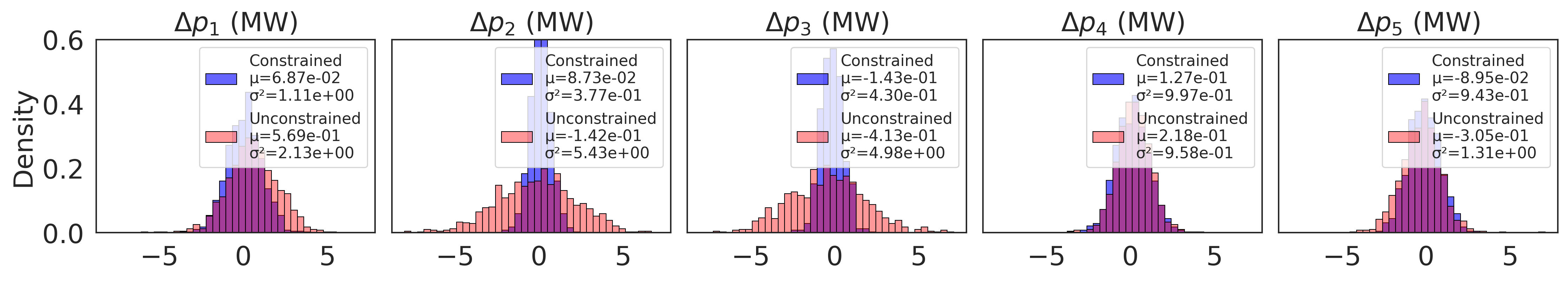} 
    }\vspace{-3mm}
    
    \subfloat{
        \includegraphics[width=0.95\linewidth]{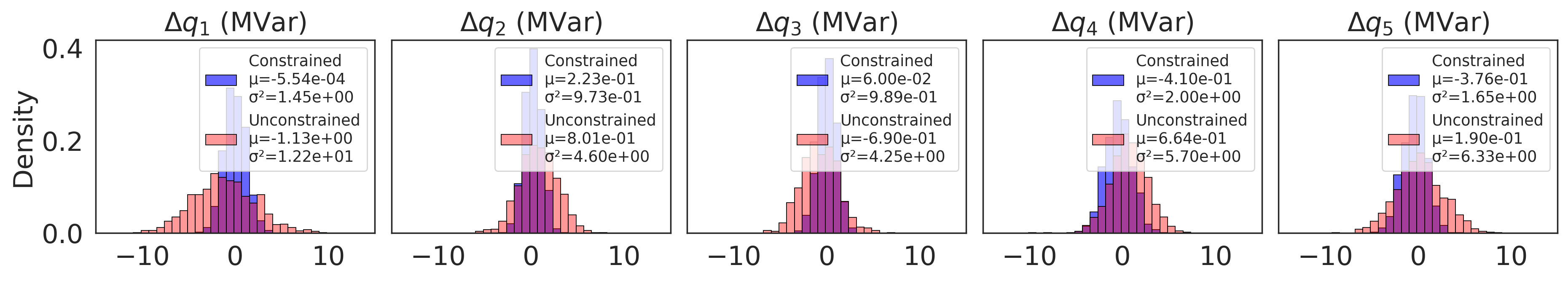}
    }
    \caption{Histograms of violation magnitudes for the active (top row) and reactive (bottom row) power balance constraints in the PJM 5-bus system, comparing synthesized data points under constrained and unconstrained sampling ($\lambda{=}10^{-2}$ vs $\lambda{=}0$).}
    \label{power_mismatch}
\end{figure*}

\begin{table}[t!]
\centering
\caption{Mean and standard deviation of power mismatches of synthesized data points for the IEEE 24-bus system under constrained and unconstrained sampling ($\lambda{=}10^{-4}$ vs $\lambda{=}0$).}
\label{tab:mismatch_24bus}
\renewcommand{\arraystretch}{1.1}
\setlength{\tabcolsep}{4pt}
\begin{tabular}{c*{4}{r}}
\toprule
\multirow{2}{*}{Bus} 
& \multicolumn{2}{c}{{$\Delta p$ (MW)}} 
& \multicolumn{2}{c}{{$\Delta q$ (MVar)}} \\
\cmidrule(lr){2-3} \cmidrule(lr){4-5}
& $\lambda{=}0$\hspace{5mm} & $\lambda{=}10^{-4}$\hspace{2.5mm}
& $\lambda{=}0$\hspace{5mm} & $\lambda{=}10^{-4}$ \hspace{2.5mm}\\
\midrule
1  & $-12.50 \pm 207.10$ & $1.03 \pm 4.80$ & $-1.20 \pm 59.80$ & $-0.60 \pm 5.50$ \\
2  & $10.10 \pm 153.00$  & $-1.50 \pm 4.90$ & $-2.56 \pm 56.20$ & $0.77 \pm 5.40$ \\
3  & $0.24 \pm 7.10$     & $-0.15 \pm 2.80$ & $-0.06 \pm 2.60$ & $-0.08 \pm 1.00$ \\
4  & $-0.30 \pm 9.10$    & $-0.03 \pm 2.80$ & $0.04 \pm 1.40$  & $-0.03 \pm 1.00$ \\
5  & $0.53 \pm 6.10$     & $0.05 \pm 4.00$  & $-0.04 \pm 2.80$ & $-0.04 \pm 1.40$ \\
6  & $0.04 \pm 26.40$    & $-0.85 \pm 3.20$ & $-0.19 \pm 12.80$ & $0.11 \pm 1.00$ \\
7  & $-0.94 \pm 17.80$   & $-0.41 \pm 3.50$ & $0.30 \pm 6.60$  & $0.07 \pm 1.40$ \\
8  & $1.00 \pm 22.60$    & $0.39 \pm 4.60$  & $-0.89 \pm 24.80$ & $-0.15 \pm 1.70$ \\
9  & $0.33 \pm 8.20$     & $0.41 \pm 3.90$  & $-0.36 \pm 8.30$ & $-0.14 \pm 1.00$ \\
10 & $0.85 \pm 7.30$     & $0.92 \pm 4.90$  & $-0.74 \pm 17.50$ & $-0.21 \pm 1.00$ \\
11 & $-0.38 \pm 7.30$    & $-0.21 \pm 2.60$ & $0.01 \pm 2.20$  & $0.05 \pm 1.40$ \\
12 & $-0.92 \pm 24.80$   & $-0.15 \pm 2.20$ & $-0.06 \pm 5.00$ & $-0.02 \pm 1.00$ \\
13 & $0.02 \pm 2.20$     & $0.02 \pm 1.40$  & $0.01 \pm 4.50$  & $0.16 \pm 2.60$ \\
14 & $-0.65 \pm 19.0$   & $-0.09 \pm 3.50$ & $0.02 \pm 2.20$  & $-0.17 \pm 1.70$ \\
15 & $-0.28 \pm 71.80$   & $-0.34 \pm 4.90$ & $-0.42 \pm 14.70$ & $-0.11 \pm 1.70$ \\
16 & $6.97 \pm 189.00$   & $0.22 \pm 3.90$ & $-1.84 \pm 54.70$ & $0.11 \pm 1.70$ \\
17 & $-1.04 \pm 19.60$   & $-0.23 \pm 4.40$ & $-0.48 \pm 19.40$ & $0.07 \pm 1.00$ \\
18 & $-1.35 \pm 94.20$   & $0.48 \pm 3.00$ & $-0.28 \pm 3.60$ & $-0.10 \pm 0.00$ \\
19 & $-2.17 \pm 78.00$   & $-0.17 \pm 4.80$ & $-0.20 \pm 5.70$ & $0.05 \pm 1.00$ \\
20 & $-3.14 \pm 55.60$   & $-0.55 \pm 4.00$ & $0.30 \pm 3.70$  & $0.03 \pm 1.00$ \\
21 & $2.00 \pm 85.80$    & $0.16 \pm 3.90$ & $-0.64 \pm 21.30$ & $0.07 \pm 1.00$ \\
22 & $-0.98 \pm 26.90$   & $-0.22 \pm 3.30$ & $-0.13 \pm 4.00$ & $0.00 \pm 0.00$ \\
23 & $3.23 \pm 68.60$    & $0.51 \pm 4.10$ & $-0.52 \pm 14.60$ & $-0.04 \pm 1.00$ \\
24 & $0.30 \pm 7.80$     & $0.06 \pm 2.80$ & $0.09 \pm 4.70$  & $0.04 \pm 1.00$ \\
\bottomrule
\end{tabular}
\end{table}

\begin{figure*}
    \centering
    \includegraphics[width=1\linewidth]{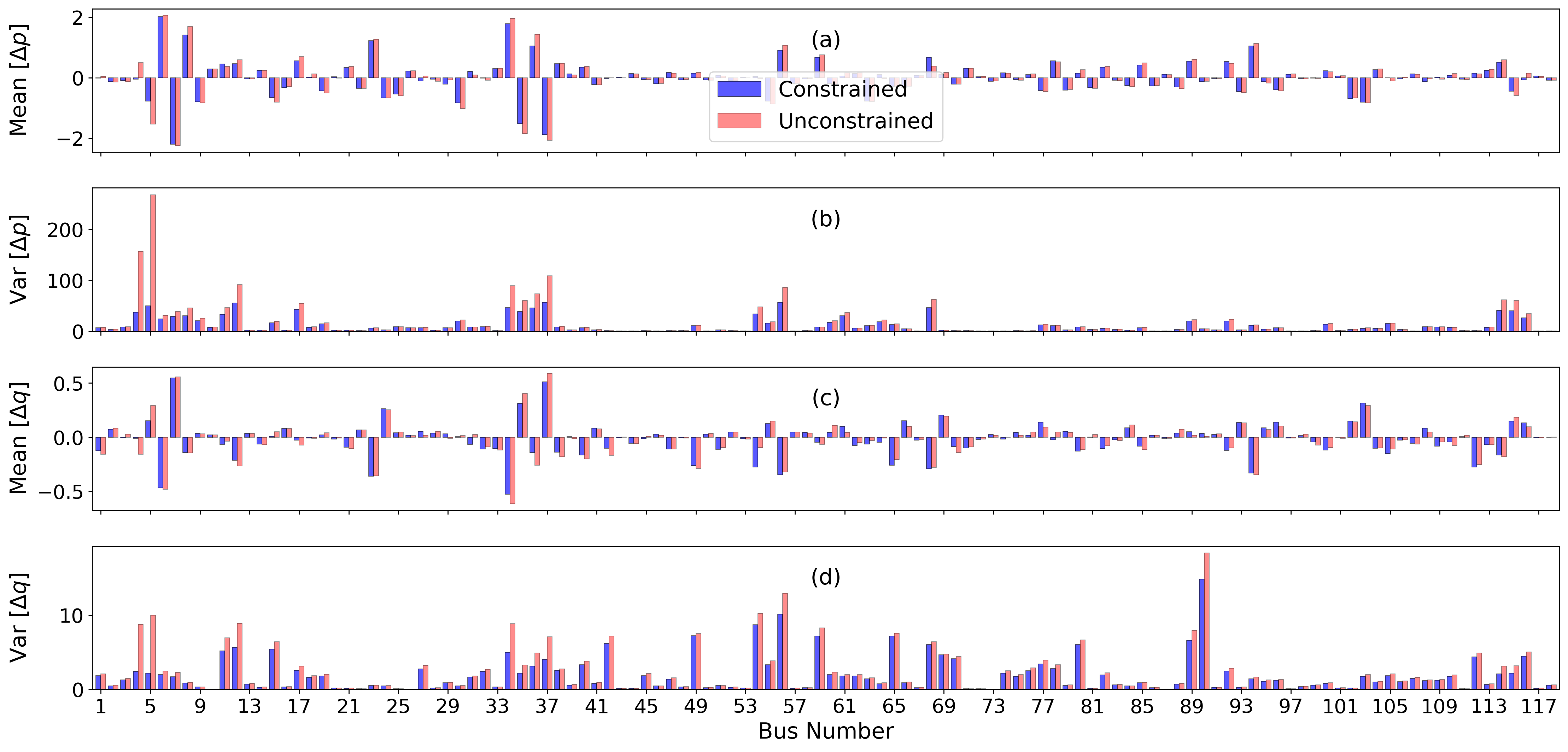}
    \caption{Comparison of per-bus (a) mean ($\mathrm{MW}$) and (b) variance ($\mathrm{MW}^2$) of the active power mismatches, and (c) mean ($\mathrm{MVar}$) and (d) variance ($\mathrm{MVar}^2$) of the reactive power mismatches for the synthesized data on the IEEE 118-bus system under constrained and unconstrained sampling ($\lambda{=}5\times10^{-4}$ vs $\lambda{=}0$).}
    \label{fig:pw_mismatch_118}
\end{figure*}

\subsection{Utility of Synthetic Data in Downstream ML Task}\label{results:downstream_task}
As an additional measure of quality, we study how well the synthesized power flow data performs in a downstream learning task. Specifically, we examine whether constrained synthetic data better supports the learning of efficient warm-start for the Newton--Raphson power flow solver \cite{molzahn2019survey}. We train a neural network $f: \mathbb{R}^{d_k} \rightarrow \mathbb{R}^{d_u}$ that maps the known inputs $\mathbf{x} \in \mathbb{R}^{d_k}$ (e.g., $\mathbf{p}$, $\mathbf{q}$, $\mathbf{v}$, or $\boldsymbol{\theta}$, depending on bus types) to the corresponding unknown outputs $\mathbf{y} \in \mathbb{R}^{d_u}$. The dimensions $d_k$ and $d_u$ depend on the specific test case. We generate two synthetic training datasets of equal size under constrained and unconstrained sampling. Using each dataset, we train a separate neural network to predict $\mathbf{y}$ from $\mathbf{x}$. We use a fully connected feedforward neural network architecture. More advanced architectures could improve accuracy but would not change the relative comparison between datasets. The performance of the models is then evaluated on a common test dataset. The results in Table \ref{tab:powerflow_mse} show that models trained on constrained synthetic data consistently yield smaller active and reactive power mismatches than those trained on unconstrained data across all the test cases. Although the gap between synthetic and ground-truth data remains, enforcing the power flow constraints during data generation clearly leads to predictions that better respect the underlying physics of power systems.

\begin{table}[h!]
\centering
\setlength{\tabcolsep}{4pt} 
\caption{Mean and standard deviation of total active and reactive power mismatches across all buses for models trained on ground truth, constrained synthetic, and unconstrained synthetic datasets (p.u.).}
\label{tab:powerflow_mse}
\begin{tabular}{@{}llcc@{}} 
\toprule
Test Case & Training Data & Mean $\pm$ Std. $|\Delta {p}|$ & Mean $\pm$ Std. $|\Delta {q}|$  \\
\midrule
\multirow{3}{*}{5-Bus} 
  & Ground Truth      & 0.0124 $\pm$ 0.0173 & 0.0242 $\pm$ 0.0469 \\
  & Constrained Syn.     & 0.0200 $\pm$ 0.0173 & 0.0434 $\pm$ 0.0944 \\
  & Unconstrained Syn.   & 0.0293 $\pm$ 0.0282 & 0.0557 $\pm$ 0.0916 \\
\midrule
\multirow{3}{*}{24-Bus} 
  & Ground Truth      & 0.1733 $\pm$ 0.0068 & 0.0415 $\pm$ 0.0003 \\
  & Constrained Syn.     & 0.3425 $\pm$ 0.0143 & 0.3093 $\pm$ 0.0376 \\
  & Unconstrained Syn.   & 0.4522 $\pm$ 0.0292 & 0.3444 $\pm$ 0.0827 \\
\midrule
\multirow{3}{*}{118-Bus} 
  & Ground Truth      & 1.6425 $\pm$ 0.0273 & 0.5838 $\pm$ 0.0066 \\
  & Constrained Syn.     & 3.9106 $\pm$ 0.0991 & 1.5219 $\pm$ 0.0311 \\
  & Unconstrained Syn.   & 4.0596 $\pm$ 0.1117 & 1.6318 $\pm$ 0.0405 \\
\bottomrule
\end{tabular}
\end{table}
\section{Conclusion}\label{conclusion}
This paper aims to generate statistically representative and physically consistent synthetic power flow datasets. We develop a diffusion model that integrates the AC power flow constraints into the data generation process through manifold-constrained gradient guidance. Numerical experiments on IEEE benchmark systems show that the model produces synthetic datasets with high statistical similarity to real data while achieving high physical feasibility. The proposed method can serve as a practical tool for system operators to generate high-quality power flow data suitable for public release and capable of supporting a wide range of downstream ML applications.

\bibliographystyle{IEEEtran}
\bibliography{mybib.bib}
\balance

\end{document}